\documentclass[letterpaper, 10 pt, conference]{ieeeconf}  

\IEEEoverridecommandlockouts                              
\usepackage{graphics} 
\usepackage{epsfig} 
\usepackage{times} 
\usepackage{amsmath} 
\usepackage{amssymb}  
\usepackage{amsfonts}
\usepackage{cite}
\usepackage{xcolor}
\usepackage{url}

\usepackage{amsthm}
\newtheorem{thm}{Theorem}
\newtheorem{prop}{Proposition}

\theoremstyle{remark}
\newtheorem{defn}{Definition}
\newtheorem{prob}{Problem}
\newtheorem{remark}{Remark}
\newtheorem{example}{Example}
\newtheorem{corollary}{Corollary}

\newcommand{\R}{\mathbb{R}}

\newcommand{\X}{\mathcal{X}}
\newcommand{\Z}{\mathcal{Z}}
\newcommand{\Xo}{\mathcal{X}_{\text{unsafe}}}
\newcommand{\Xf}{\mathcal{X}_{\text{safe}}}
\newcommand{\Zf}{\mathcal{Z}_{\text{safe}}}
\newcommand{\Df}{\mathcal{D}_{\text{safe}}}
\newcommand{\Do}{\mathcal{D}_{\text{unsafe}}}
\newcommand{\xgoal}{x_\star}
\newcommand{\zgoal}{z_\star}

\title{\LARGE \bf  Goal-Conditioned Neural ODEs with Guaranteed Safety and Stability\\ for Learning-Based All-Pairs Motion Planning}

\author{Dechuan Liu, Ruigang Wang and Ian R. Manchester
\thanks{*This work was supported in part by the Australian Research Council through projects DP230101014 and IH210100030.}
\thanks{The authors are with the Australian Centre for Robotics (ACFR), and the School of Aerospace, Mechanical and Mechatronic Engineering, The University of Sydney, Sydney, NSW 2006, Australia
        {\tt\small ruigang.wang@sydney.edu.au}}%
}

\begin{document}

\maketitle
\thispagestyle{empty}
\pagestyle{empty}

\begin{abstract}
This paper presents a learning-based approach for all-pairs motion planning, where the initial and goal states are allowed to be arbitrary points in a safe set. We construct smooth goal-conditioned neural ordinary differential equations (neural ODEs) via bi-Lipschitz diffeomorphisms. Theoretical results show that the proposed model can provide guarantees of global exponential stability and safety (safe set forward invariance) regardless of goal location. Moreover, explicit bounds on convergence rate, tracking error, and vector field magnitude are established. Our approach admits a tractable learning implementation using bi-Lipschitz neural networks and can incorporate demonstration data. We illustrate the effectiveness of the proposed method on a 2D corridor navigation task.
\end{abstract}

\section{Introduction}

Motion planning in environments with complex geometric constraints -- arising from obstacles, workspace boundaries, and configuration space structure -- remains a fundamental challenge in robotics. 

Classical sampling-based methods such as probabilistic roadmaps (PRM) \cite{Kavraki1996ProbabilisticRF} and rapidly-exploring random trees (RRT) \cite{LaValle1998RapidlyexploringRT} are simple to implement and widely used, but suffer from certain limitations. In particular, they are inherently finite in their representation: RRTs are generated from a fixed start position (single-query) while PRMs can handle a finitely-many sampled start and end positions (multi-query) \cite{lavalle2006planning}. They do not define a smooth feedback policy from all possible start and end positions (all-pairs). Furthermore, they are inherently model-based and difficult to adapt to learning-based methodologies such as learning from demonstration (LfD) \cite{ravichandar2020recent}.

An emerging paradigm represents the desired motion via a continuous dynamical system whose solutions serve directly as executable trajectories, see e.g. \cite{ab2011gmm, billardLearningAdaptiveReactive2022, pmlr-v162-zhi22a,jinLearningFlexibleNeural2023c,guptaCompactOneshotModelling2026}. This approach has the advantage that it effectively defines a feedback policy, so it is robust, adaptable, and real-time implementable, and compatible with learning frameworks such as LfD when the dynamical models are parameterized, e.g. neural ordinary differential equations (neural ODEs).

The dynamical systems approach to robot motion generation can be traced back to classical approaches such as potential fields \cite{Khatib1985Real} which combine attractive and repulsive forces but can suffer from local minima \cite{Tilove1989localminimum} and instability in narrow passages \cite{Koren1991PotentialFM}. Mitigating these issues remains an ongoing research activity \cite{Huber2024avoidance}. Navigation function, introduced by Koditschek and Rimon \cite{koditschek1990robot}, provide a theoretical foundation that is closely related to the concept of a Lyapunov function. They provide a construction of for simplified ``sphere worlds'', and showed how they can in principle be adapted to more complex but topologically-equivalent spaces via diffeomorphisms, however at the time constructive methods were lacking.

\begin{figure}[!t]
    \centering
    \includegraphics[width=\linewidth]{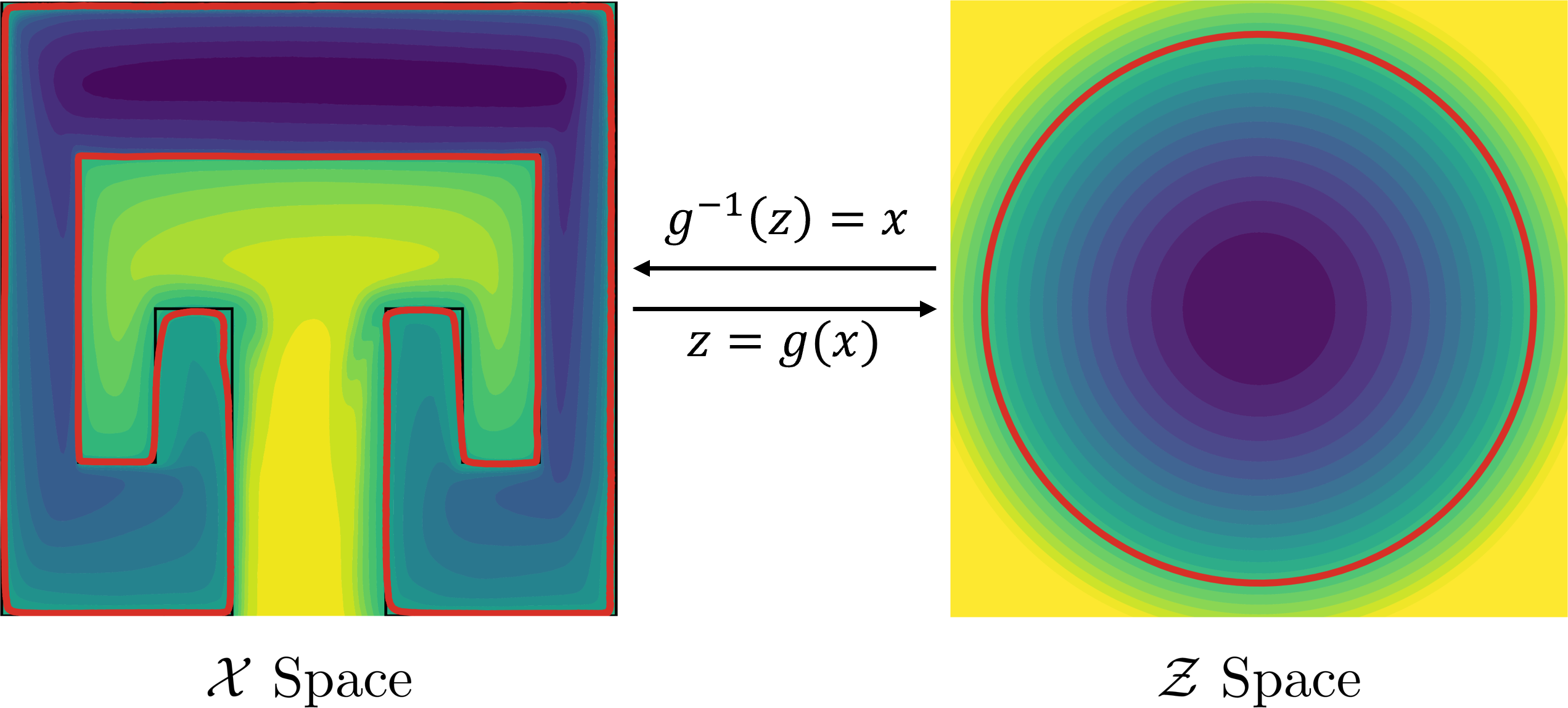}
    \centering
    \caption{
    Our approach is based on learning a bi-Lipschitz diffeomorphism $g$ that maps a geometrically complex safe set $\Xf$ in the $\X$-space (left) onto the unit ball in the $\Z$-space (right). Then simple straight-line point-to-point motions in $\Z$-space can be smoothly pulled back to $\X$ space, defining a goal-conditioned neural ODE which guarantees stability and safety and takes the form of a natural gradient flow.
    }
    \label{fig:boundary_mapping}
\end{figure}

The central challenge therefore is designing (or learning) a dynamical system with the required properties: it should be sufficiently smooth, flexible enough to reproduce the desired task behavior, have some (preferably global) stability properties, and provide the ability to avoid obstacles or other unsafe regions. 

Lyapunov stability theory provides a principled framework for designing stable dynamical systems.
The classical work of Wilson \cite{wilson1967structure}, combined with the resolution of the generalized Poincaré Conjecture \cite{anderson2004geometrization} established that all Lyapunov functions have level sets homeomorphic to spheres, and diffeomorphic for all dimensions other than five. Such results are closely related to the problem of global linearization of nonlinear systems, see \cite{kvalheimGlobalLinearizationHyperbolicity2025} and references therein. This naturally suggests parameterizing Lyapunov functions as the composition of a diffeomorphism and a simple quadratic \cite{wang2024monotone, Cheng2024LearningSA}.

Early work ensuring stability in the dynamical systems approach to motion planning relied on quadratic Lyapunov functions, limiting flexibility \cite{ab2011gmm}. Increasing flexibility via state diffeomorphisms has been explored from several directions. To our knowledge it was first applied in the context of stable nonlinear system identification:  
\cite{tobenkin2010convex, tobenkin2017convex} learned contracting dynamics via polynomial diffeomorphisms of the state space, enforced via sum-of-squares programming. The idea was explicitly applied to robot motion planning in \cite{NEUMANN20151}, by integrating diffeomorphisms with the method of \cite{ab2011gmm}, however the class of diffeomorphisms was quite limited. A more flexible class of diffeomorphisms based on Gaussian radial basis function kernels was investigated in \cite{PERRIN201651}. 

More recently, neural network methods based on normalizing flows have been developed \cite{rana2020euclideanizing} and extended to limit cycles \cite{pmlr-v162-zhi22a} and adaptation to environmental changes
\cite{pmlr-v168-zhi22a}. The present paper builds most directly on\cite{Cheng2024LearningSA}, which proposed a class of exponentially stable neural dynamics constructed using \textit{bi-Lipschitz} neural networks \cite{wang2024monotone}. This model class provides not only certified stability but explicit bounds on the rate of stability and potential overshoot, as well as fast splitting-based algorithms for inversion.

Robotics includes many potentially safety-critical or mission-critical applications, and certifying safety of learning-based methods is a major area of current research (see, e.g., the reviews \cite{brunke2022safe, manchester2026neural} and references therein). In the context of the dynamical systems approach to motion planning, a recent work \cite{nawaz2024learning} proposed combining learned neural ODEs with control Lyapunov functions (CLFs) and control barrier functions (CBFs). On the other hand, it is known that existence of a CLF and a CBF separately does not guarantee existence of a compatible CLF-CBF pair \cite{mestres2025converse}, which complicates the learning setup.

In the aforementioned approaches to ensuring stability and safety in dynamical systems LfD, it is generally assumed that the goal-state of the motion is fixed before the learning process, i.e. it is a single-query setup. Changes to the goal state can sometimes be incorporated, but would generally require retraining and/or fresh certification of stability and safety of the resulting motion.

{\bf Contributions.} In this paper, we propose a class of goal-conditioned neural dynamical systems that incorporate built-in guarantees of global exponential stability and safety (safe set forward invariance), for \textit{all} combinations of initial state and goal state within the safe set. The models are compatible with learning-based paradigms for motion planning, such as learning from demonstration. The main contributions are: (1) a systematic approach for constructing goal-conditioned dynamical systems via diffeomorphisms; (2) theoretical results establishing guarantees of safety and exponential stability regardless of goal location; (3) a tractable machine learning formulation for a diffeomorphism; (4) empirical validation confirming the effectiveness of the proposed approach.

{\bf Notation. } A mapping $f:\R^n\rightarrow\R^m$ is said to be of class $C^k$ if it has up to $k$th continuous derivatives. A continuously differentiable mapping $g:\R^n\rightarrow\R^n$ is called a diffeomorphism if it is a bijection and its inverse $g^{-1}$ is also differentiable. Given a $C^1$ function $V:\R^n \rightarrow \R$, its gradient is taken as $\nabla V:=\bigl(\partial V/\partial x\bigr)^\top$. We denote unit ball as $\mathcal{B}^n=\{x\in \R^n:|x|\leq 1\}$, where $|\cdot|$ is the Euclidean norm. Given a set $X\subset\R^n$, we use $\partial X$ and $\mathrm{Int}(X)$ to denote its boundary and interior, respectively.

\section{Preliminaries and Problem Formulation}\label{sec:problem setup}

\subsection{The Dynamic Approach for Motion Planning}

The basic motion planning problem in robotics usually considers the robot dynamics to be \textit{fully-actuated} and \textit{velocity-controlled}, i.e.
\begin{equation}\label{eq:robot-dyn}
    \dot{x}(t)=u(t).
\end{equation}
where $x(t)\in \X \subseteq \mathbb{R}^n$ is the state, e.g., $x(t)$ could be the position of a robot arm's end effector.

While the dynamics are very simple, the difficulty comes from two sources. Firstly, the requirement for collision avoidance and other safety constraints, which can be represented as
\begin{equation} 
   x(t)\in \Xf \quad \forall\, t\geq 0
\end{equation}
where $\Xf\subset \X$ is a safe set that may have complex geometry. We denote by $\Xo$ its complement $\Xo = \X\setminus \Xf$, i.e., the set of unsafe states.

Secondly, the motion task may be complex and only partially specified. While it usually includes motion towards a goal position $\xgoal$, among the infinite variety of possible motions approach the same goal, the desirable ones may be specified only indirectly via a limited set of demonstration data. The robot motion should not only accurately reproduce the training demonstrations, but also generalize to new conditions and react gracefully to disturbances. This generally requires some form of \textit{smoothness} and \textit{stability} of the dynamics.

\subsection{Problem Statement}
In this work, we focus on a learning-based \emph{all-pairs} motion planning problem, where both $x_0$ and $\xgoal$ are allowed to be arbitrary points in $\Xf$. Specifically, we aim to learn a smooth goal-conditioned dynamical system of the form
\begin{equation} \label{eq:system}
    \dot{x}(t) = f(x(t), \xgoal),\quad x(0)=x_0,
\end{equation}
where $f(\xgoal,\xgoal)=0$ for all $\xgoal\in \Xf$, i.e. the goal state is an equilibrium. 

To formalize the desired properties of \eqref{eq:system}, we first recall the following standard definition:
\begin{defn}\label{def:forward_inv}
    For a given dynamical system with state $x(t)$, a set $\mathcal S$ is called \textit{forward invariant} if $x(0)\in \mathcal S$ implies $x(t)\in\mathcal S$ for all $t\geq 0$. 
\end{defn}
We use the following notions of safety and stability for a goal-conditioned system:
\begin{defn}\label{def:safe}
System \eqref{eq:system} is called \emph{safe} w.r.t. the set $\Xf$ if for any goal state $\xgoal \in \Xf$, the set $\Xf$ is forward invariant.
\end{defn}
\begin{defn}[\cite{hines2011equilibrium}]\label{def:exp-convergence}
    System \eqref{eq:system} is globally \emph{equilibrium-independent exponentially stable} if for any initial state $x_0\in\X$ and any equilibrium $\xgoal\in \X$, the solution $x(t)$ satisfies
    \begin{equation*}
        |x(t) - \xgoal| \leq \kappa e^{-\lambda t} |x_0 - \xgoal|, \quad \forall t \geq 0,
    \end{equation*}
    for some $\kappa\geq 1$ and $\lambda>0$. 
\end{defn} 

Here we are interested in the following problem. 
\begin{prob}
   Given training data characterising the safe and unsafe sets:
\begin{equation}\label{eq:datasets}
    \begin{split}
        \mathcal{D}_{\text{safe}}&=\{x_i: x_i\in \Xf\}_{1\leq i\leq M}, \\
        \mathcal{D}_{\text{unsafe}}&=\{x_j: x_j\in \Xo\}_{1\leq j\leq N},
    \end{split}
\end{equation}
and additionally some task-relevant data for the system's desired behaviour \textit{inside} the safe set, e.g. demonstration data:
\begin{equation}
    \mathcal{D}_{\text{demo}}=\{(x_{k}, x_{k,\star}, \dot{x}_{k}): x_k,x_{k,\star}\in \Xf\}_{1\leq k\leq K},
\end{equation}
the goal is to learn a smooth dynamical system of the form \eqref{eq:system} with the following properties:
\begin{itemize}
\item The system \eqref{eq:system} is safe w.r.t. the set $\Xf$.
\item The system \eqref{eq:system} has a known bound on velocity on the safe set: $|f(x,\xgoal)|\le B$ for all $x,\xgoal\in\Xf$.
\item The system \eqref{eq:system} is globally equilibrium-independent exponentially stable.
\item The system \eqref{eq:system} effectively mimics the demonstration data $\mathcal{D}_{\text{demo}}$ inside $\Xf$, or otherwise meets the training objectives, and generalizes smoothly.
\end{itemize}
\end{prob}

We note that the first three objectives can be considered hard constraints, with the caveat that the first requirement depends on the extent that the data sets $\mathcal{D}_{\text{safe}}$ and $\mathcal{D}_{\text{unsafe}}$ accurately represent the true sets $\Xf, \Xo$. 

The fourth requirement is somewhat loose, but will generally be supported by having a sufficiently \textit{flexible} class of models to meet the training objective while ensuring satisfaction of the first three requirements, as well as some possibility to tune the smoothness of the model.

\subsection{Preliminaries on Bi-Lipschitz Diffeomorphisms}

To formulate our approach, we first require some technical machinery. We extensively utilise bi-Lipschitz diffeomorphisms: 
\begin{defn}
\label{def:bi-lip}
    The diffeomorphism $g:\R^n\to\R^n$ is said to be \emph{bi-Lipschitz} if for all $x_1, x_2\in \R^n$, we have
    \begin{equation}
    \label{eq:bi-lipschitz}
        \mu|x_1-x_2|\leq |g(x_1)-g(x_2)|\leq \nu |x_1-x_2|
    \end{equation}
    for some $\nu \geq \mu >0$.
\end{defn} 
Note that $g^{-1}$ is also a bi-Lipschitz diffeomorphism as 
\begin{equation}\label{eq:bi-lipschitz-inverse}
    \frac{1}{\nu} |z_1-z_2|\leq \bigl|g^{-1}(z_1)-g^{-1}(z_2)\bigr|\leq \frac{1}{\mu} |z_1-z_2|.
\end{equation}
Moreover, $g$ also induces a Riemannian metric
\begin{equation}\label{eq:metric}
    M(x):=G(x)^\top G(x)
\end{equation}
with $G(x)=\partial g(x)/\partial x$ as the Jacobian of $g$ at $x$, where $M$ gives a
notion of local distance \cite{boothby2003introduction}. Since $g$ is bi-Lipschitz, $M$ is uniformly bounded, i.e.,
\begin{equation}\label{eq:uniform-bound}
    \mu^2 I\preceq M(x)\preceq \nu^2I, \quad \forall x\in \R^n.
\end{equation}

\section{Main Theoretical Results}\label{sec:main results}
In this section, we provide a systematic approach to constructing an all-pairs motion planner of the form \eqref{eq:system} using bi-Lipschitz diffeomrophisms.

\subsection{All-Pairs Motion Planning via Natural Gradient Flow}
Let $g:\R^n\rightarrow\R^n$ be a bi-Lipschitz diffeomorphism. We take the candidate Lyapunov function as follows
\begin{equation} \label{eq:v-func}
    V(x,\xgoal)= \tfrac{\lambda}{2}|g(x)-g(\xgoal)|^2.
\end{equation}
where $\lambda>0$ is a tunable parameter.
We then construct the dynamics \eqref{eq:system} based on the  natural gradient flow of $V(x,\xgoal)$, i.e., 
\begin{equation}\label{eq:natural_gradient}
    \dot{x} =f(x,\xgoal):= -M(x)^{-1} \nabla_x V(x, \xgoal).
\end{equation}
The matrix inverse is
 well-defined due to \eqref{eq:uniform-bound},  and the system has a unique equilibrium point at $\xgoal$. Our main theoretical result is as follows.
\begin{thm}\label{thm:main}
    If there exists a bi-Lipschitz diffeomorphism $g:\Xf\rightarrow \mathcal{B}^n$, then the following statements hold:
    \begin{enumerate}
        \item System \eqref{eq:natural_gradient} is safe w.r.t. the set $\Xf$.
        \item The vector field in \eqref{eq:natural_gradient} is bounded 
        \begin{equation}\label{eq:vel_bound}
            |f(x,\xgoal)|\leq \frac{2\lambda}{\mu}.
        \end{equation}
        \item System \eqref{eq:natural_gradient} is globally  equilibrium-independent exponentially stable.
    \end{enumerate}
\end{thm}

\begin{proof}
As shown in \cite{rana2020euclideanizing}, under the coordinate transformation $z=g(x)$, system \eqref{eq:natural_gradient} is equivalent to
\begin{equation}\label{eq:z_dynamics}
    \dot{z}=\lambda(\zgoal-z)
\end{equation}
with $\zgoal=g(\xgoal)$. Thus, system \eqref{eq:natural_gradient} has explicit solutions of
\begin{equation}\label{eq:xt}
    x(t)=g^{-1}(z(t))
\end{equation}
where
\begin{equation}\label{eq:zt}
    z(t)=z_0e^{-\lambda t}+z_\star \bigl(1-e^{-\lambda t}\bigr)
\end{equation}
with $z_0=g(x_0)$. 

Statement 1): The main idea is that under the diffeomorphism $g$, the transformed safe set $\Zf:=\mathcal{B}^n$ which is convex, and from \eqref{eq:zt} we have that $z(t)$ is a straight line from $z_0$ to $\zgoal$. Hence for any $z_0$ and $\zgoal$ in $\Zf$ the path between them remains in $\Zf$, i.e. $z(t)\in \Zf$ for all $t\geq 0$. Passing back to $x(t)=g^{-1}(z(t))$, this implies that  $x(t)\in \Xf$ for all $t\geq 0$.  

Statement 2): The $\Z$-space dynamics \eqref{eq:z_dynamics} is a push forward map of \eqref{eq:natural_gradient}:
\begin{equation}\label{eq:push-forward}
    G(x)f(x,x_\star)=\lambda(z_{\star}-z),
\end{equation}
which further implies
\begin{equation}\label{eq:vf-bound}
    |f(x,x_\star)|\leq \lambda \|G(x)^{-1}\| |z-z_\star|  \leq \frac{2\lambda}{\mu}
\end{equation}
since $z,z_\star\in \mathcal{B}^n$ and $g$ has lower Lipschitz bound of $\mu$.

Statement 3): Since $g$ is bi-Lipschitz, we obtain  
\begin{equation}\label{eq:x-exp-stable}
    \begin{split}
    |x(t)-\xgoal|&=|g^{-1}(z(t))-g^{-1}(\zgoal)| \leq  \frac{1}{\mu} |z(t)-\zgoal|\\
    &=\frac{1}{\mu}|z_0-z_\star|e^{-\lambda t} \leq \frac{\nu}{\mu}|x_0-\xgoal|e^{-\lambda t}.
\end{split}
\end{equation}
System \eqref{eq:natural_gradient} is equilibrium-independent exponentially stable. 
\end{proof}

\begin{remark}
    While \eqref{eq:natural_gradient} can be considered as a feedback controller to be implemented in real-time for the system \eqref{eq:robot-dyn}, there is also often need to predict future motions in simulation. For this case, we can sample the analytic solution  \eqref{eq:zt}  for $z(t)$ at a grid of points, and pass them in parallel through $g^{-1}$ to obtain $x(t)$. When a fast inverse algorithm is available $g$, as in \cite{wang2024monotone}, this can be done in parallel on a GPU.
\end{remark}

\begin{remark} It can be shown that system \eqref{eq:natural_gradient} is incrementally exponentially stable w.r.t. the incremental Lyapunov $V(x_1,x_2)$. From the contraction theory perspective (\cite{Lohmiller1998Contraction}), system \eqref{eq:natural_gradient} is contracting w.r.t. the metric $M(x)$ in \eqref{eq:metric} \cite[Thm.~1]{yi2023equivalence}, see also \cite{wensing2020beyond}. 
\end{remark}

It is clear from the proof that without any additional effort, we can extend the above theorem by replacing the set $\mathcal{B}^n$ to any convex, compact set $\Zf$ so that $\Xf$ admits more complicated shapes (e.g., with sharp corners).
\begin{corollary}
Suppose that $g:\Xf\rightarrow \Zf$ is a bi-Lipschitz diffeomorphism, where $\Zf\subset\R^n$ is a convex and compact set. Then, Statement 1) and 3) hold. For Statement 2), the vector field is bound
\[
|f(x,\xgoal)|\leq \frac{D\lambda}{\mu},
\]
where $D$ is the diameter of $\Zf$, i.e.,
\[
D:=\max_{z_1,z_2\in \Zf} |z_1-z_2|.
\]
\end{corollary}

\begin{remark}
    If $\Xf$ is not diffeomorphic to a ball, e.g., $\Xf$ contains holes, then the approach needs to be modified. The navigation function approach includes strategies for dealing with this via diffeomorphism to a ``sphere world'' in which the free space is a ball with a finite number of ball-shaped obstacles removed \cite{romon1990exact}. Under certain conditions, almost-global stability can still be certified. We leave the details of the extension for a future work.
\end{remark}

\subsection{Safety Properties}

The key property of our approach is that it guarantees safety for arbitrary start/goal pairs. In the control literature, safety in the form of forward invariance of a set is often certified by a barrier function, as defined below (see e.g. \cite{amesControlBarrierFunction2017a}):
\begin{defn}
    Consider a nonlinear system $\dot{x}=f(x,t)$. A continuously differentiable function $h:\R^n\rightarrow \R$ is called a \emph{barrier function} of $\Xf$ if there exist a class $\mathcal{K}$ function $\alpha(\cdot)$ such that 
    \begin{subequations}
    \begin{gather}
        h(x)\geq 0\quad \forall x\in \Xf, \label{eq:B-safe}\\
        h(x)<0\quad \forall x\in \X_{\text{unsafe}}, \label{eq:B-unsafe} \\
        \frac{\partial h}{\partial x}f(x,t)\geq -\alpha(h) \quad \forall x\in \X. \label{eq:B-decay}
    \end{gather}
    \end{subequations}
\end{defn}
Note that the Lyapunov function $V(x,\xgoal)$ in \eqref{eq:v-func} is not a barrier function for $\Xf$ as $V(x,\xgoal)$ may not be a constant for all $x\in \partial \Xf$. The following result gives an explicit construction of barrier function for the proposed goal-conditioned neural ODE \eqref{eq:natural_gradient}.

\begin{prop}\label{prop:barrier}
    Suppose that conditions of Theorem~\ref{thm:main} hold. Then, the forward invariance of $\Xf$ can be certified by the following barrier function 
    \begin{equation}\label{eq:barrier}
        h(x)=1- \left|g(x)\right|^2.
    \end{equation}
\end{prop}
\begin{proof}
    Since $g:\Xf\rightarrow \mathcal{B}^n$ is a bi-Lipschitz diffeomorphism, then $h(x)$ satisfies \eqref{eq:B-safe} - \eqref{eq:B-unsafe}. The time derivative of $h$ yields
    \begin{equation}
        \begin{split}
            \dot{h}&=2\lambda g(x)^\top G(x)M(x)^{-1}G(x)^\top (g(x)-g(x_\star))\\
            &=2\lambda g(x)^\top(g(x)-g(\xgoal)).
        \end{split}
    \end{equation}
    For $x\in \Xf$, we have  
    \begin{equation*}
        \dot{h}=2\lambda (|g(x)|^2-1)+2\lambda(1-g(x)^\top g(\xgoal)) \geq -2\lambda h(x)
    \end{equation*}
    where the last inequality follows by $g(x),g(\xgoal)\in \mathcal{B}^n$. For $x\in \Xo$ (i.e. $|g(x)|>1$ and $h(x)<0$), we can obtain 
    \begin{equation}
        \begin{split}
            \dot{h}&=2\lambda( |g(x)|^2-g(x)^\top g(\xgoal))\\
            &\geq 2\lambda (|g(x)|^2-|g(x)|)\geq -\lambda h(x).
        \end{split}
    \end{equation}
    
    Thus, \eqref{eq:B-decay} holds and $h(x)$ is a barrier function of $\Xf$.
\end{proof}

\subsection{Time-Varying Goal Location}

When the goal position is time-varying with uncertain but bounded velocity, e.g. to reach for a moving object, the following result shows that our approach still guarantees safety and converges to a bounded region around the goal.
\begin{thm}\label{prop: 1}
    Consider system \eqref{eq:natural_gradient} with time-varying goal $\xgoal(t)$ with $|\dot{x}_{\star}(t)|\leq b$ for all $t\geq 0$. If $g:\Xf\rightarrow\Zf$ is a bi-Lipschitz diffeomorphism, where $\Zf$ is a convex compact set, then the following statements hold:
    \begin{enumerate}
        \item[1)] The set $\Xf$ is forward invariant.
        \item[2)] The time-varying vector field is bounded
        \[
        |f(x,\xgoal(t))|\leq \frac{2\lambda}{\mu}.
        \]
        \item[3)] For any $x_0\in \Xf$, the tracking error $\epsilon(t):=x(t)-\xgoal(t)$ satisfies
        \begin{equation}
            |\epsilon(t)|\leq \frac{\nu}{\mu}\left(|\epsilon(0)|e^{-\lambda t}+\frac{b}{\lambda}\right).
        \end{equation}
    \end{enumerate} 
\end{thm}
\begin{proof}
    Statement 1): System \eqref{eq:natural_gradient} with time-varying $\xgoal(t)$ can also be transformed into $\dot{z}=\lambda(z_\star(t)-z)$ with $z_{\star}(t)=g(x_\star(t))$. Since $\Zf$ is compact and convex, then the time-varying vector field $\lambda(z_\star(t)-z)$ always points into $\Zf$ or is tangent to $\partial \Zf$ for any $z\in\partial \Zf$ and $z_{\star}(t)\in \Zf$. By Nagumo's theorem \cite{nagumo1942lage} we obtain that $\Zf$ is forward invariant, implying that $\Xf$ is forward invariant under \eqref{eq:natural_gradient}. 

    Statement 2) follows directly by \eqref{eq:push-forward} and \eqref{eq:vf-bound}. We now focus on Statement 3). First, the dynamics of $\epsilon_z(t):=z(t)-z_\star(t)$ in the $\Z$-space can be rewritten as
    \[
    \dot{\epsilon}_z=-\lambda \epsilon_z-\dot{z}_{\star}(t),
    \]
    where $|\dot{z}_\star(t)|\leq \nu b$. This implies $|\epsilon_z(t)|\leq |\epsilon_z(0)|e^{-\lambda t}+\nu b/\lambda $. Finally, following the procedure in \eqref{eq:x-exp-stable} yields
    \[
    \begin{split}
        |\epsilon(t)|\leq \frac{1}{\mu}|\epsilon_z(t)|
        \leq \frac{\nu}{\mu}\left(|\epsilon(0)|e^{-\lambda t}+\frac{b}{\lambda}\right).
    \end{split}
    \]
\end{proof}

\subsection{Finite-time Convergence via Euclidean Norm Potential}
Similar to \cite{rana2020euclideanizing}, the proposed Lyapunov function $V(x,\xgoal)$ can incorporate a more general potential function $\Phi:\R^n\rightarrow \R$, i.e.,
\begin{equation}
    V(x,\xgoal)=\Phi(g(x)-g(\xgoal)).
\end{equation}
Then, the dynamics of \eqref{eq:natural_gradient} in the $\Z$-space becomes
\begin{equation}\label{eq:z-phi}
    \dot{z}=-\nabla_z \Phi(z-\zgoal)
\end{equation}
which can be pulled back in to $\X$ space.
If $\Phi$ continuously differentiable and satisfies the \emph{Polyak-\L{}ojasiewicz} (PL) condition \cite{polyak1963gradient,lojasiewicz1963topological}, then global exponential stability can still be established. 

If finite-time convergence is desired, then  $\Phi$ can be taken as the Euclidean norm (i.e., $\Phi(z)=\lambda|z|$), instead of the norm squared, and then system \eqref{eq:z-phi} becomes
\begin{equation}
    \dot{z}=-\lambda \frac{z-\zgoal}{|z-\zgoal|},
\end{equation}
although the dynamics is not smooth at the goal $\zgoal$. Since the resulting trajectory has unit velocity in $\Z$ space, we can, analogously to Theorem~\ref{thm:main}, obtain simple upper and lower bounds on the vector field velocity in the $\X$-space, i.e.,
\begin{equation}
    \frac{\lambda}{\nu} \leq  |f(x,\xgoal)|\leq \frac{\lambda}{\mu}, \; x\neq \xgoal.
\end{equation}

\subsection{Natural Gradient Flow {\it v.s.} Gradient Flow}
A natural question for the proposed approach is: what advantages does natural gradient flow offer over standard gradient flow? Our answer is as follows: standard gradient flow does not provide safety guarantees when $\xgoal$ varies.

Given a Lyapunov function $V(x,\xgoal)$ in \eqref{eq:v-func}, we consider the following gradient flow dynamics:
\begin{equation}\label{eq:grad-flow}
    \dot{x}=-\nabla_x V(x, \xgoal).
\end{equation}
From \cite[Thm.~1]{Cheng2024LearningSA}, we can conclude that the above system achieves equilibrium-independent exponential stability. However, it cannot provide safety guarantees for all $\xgoal\in \Xf$, see the example below.

\begin{example}
\label{ex:shear_transformation}
Consider the mapping $g:\R^2\to\R^2$ defined by 
\begin{equation}\label{eq:g analytical}
    g(x) =
    \begin{bmatrix} 1 & 0 \\ h(x_1) & 1 \end{bmatrix}
    \begin{bmatrix} x_1 \\ x_2 \end{bmatrix}
\end{equation}
where $h(x_1) = 2\sin(x_1) + \cos(5x_1) - 3x_1$. It is a bi-Lipschitz diffeomorphism on any compact set $\mathcal K\subset\R^2$ with $g^{-1}$ defined by $x_1=z_1$ and $x_2=z_2-h(z_1)z_1$. We take $\Zf = \mathcal{B}^2$ and $\Xf = g^{-1}(\Zf)$. 

\begin{figure}[!tb]
    \centering
    \begin{tabular}{c}
        \includegraphics[width=0.85\linewidth]{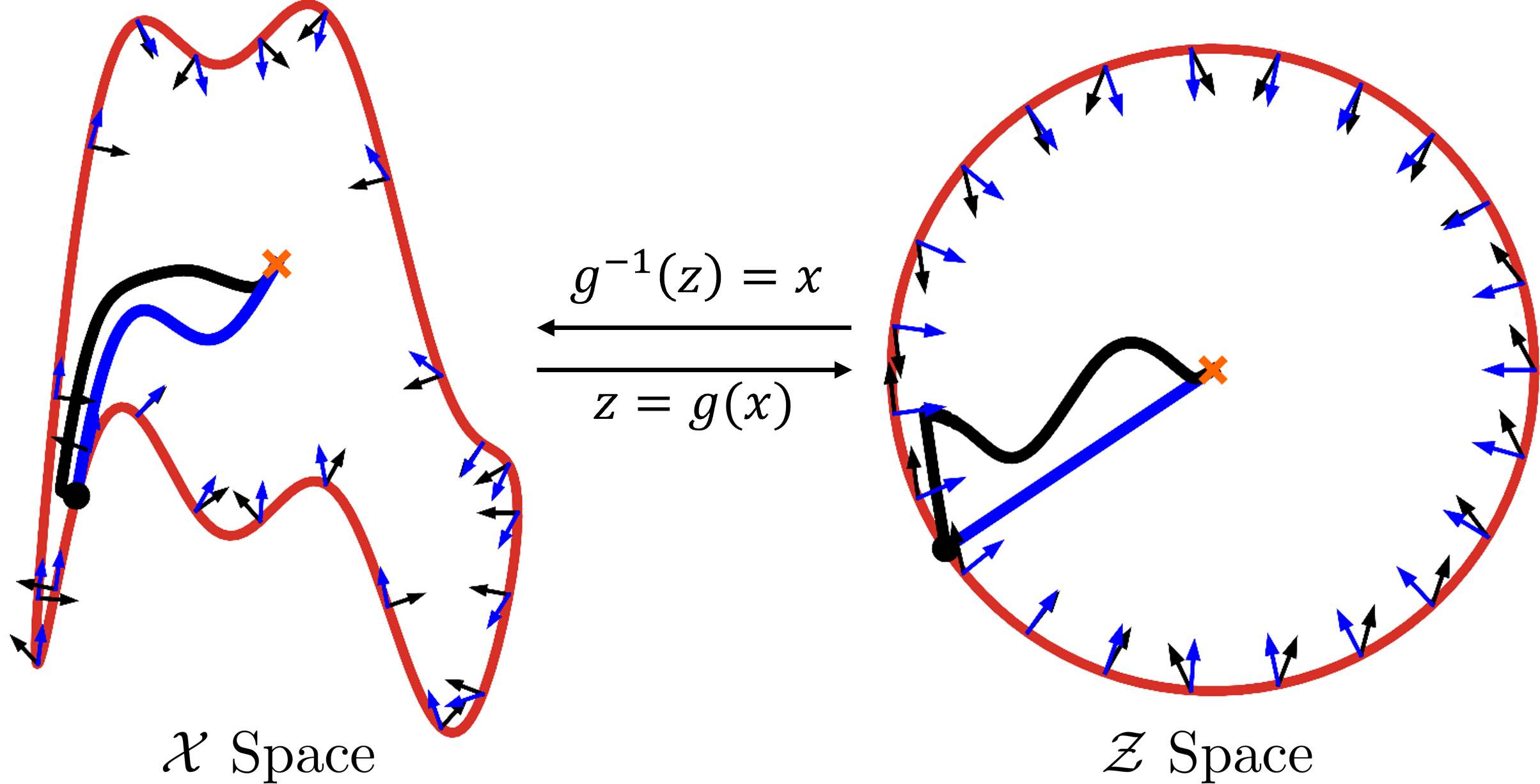} \\
        (a) $\xgoal=(0, 0)$ \\
        \includegraphics[width=0.85\linewidth]{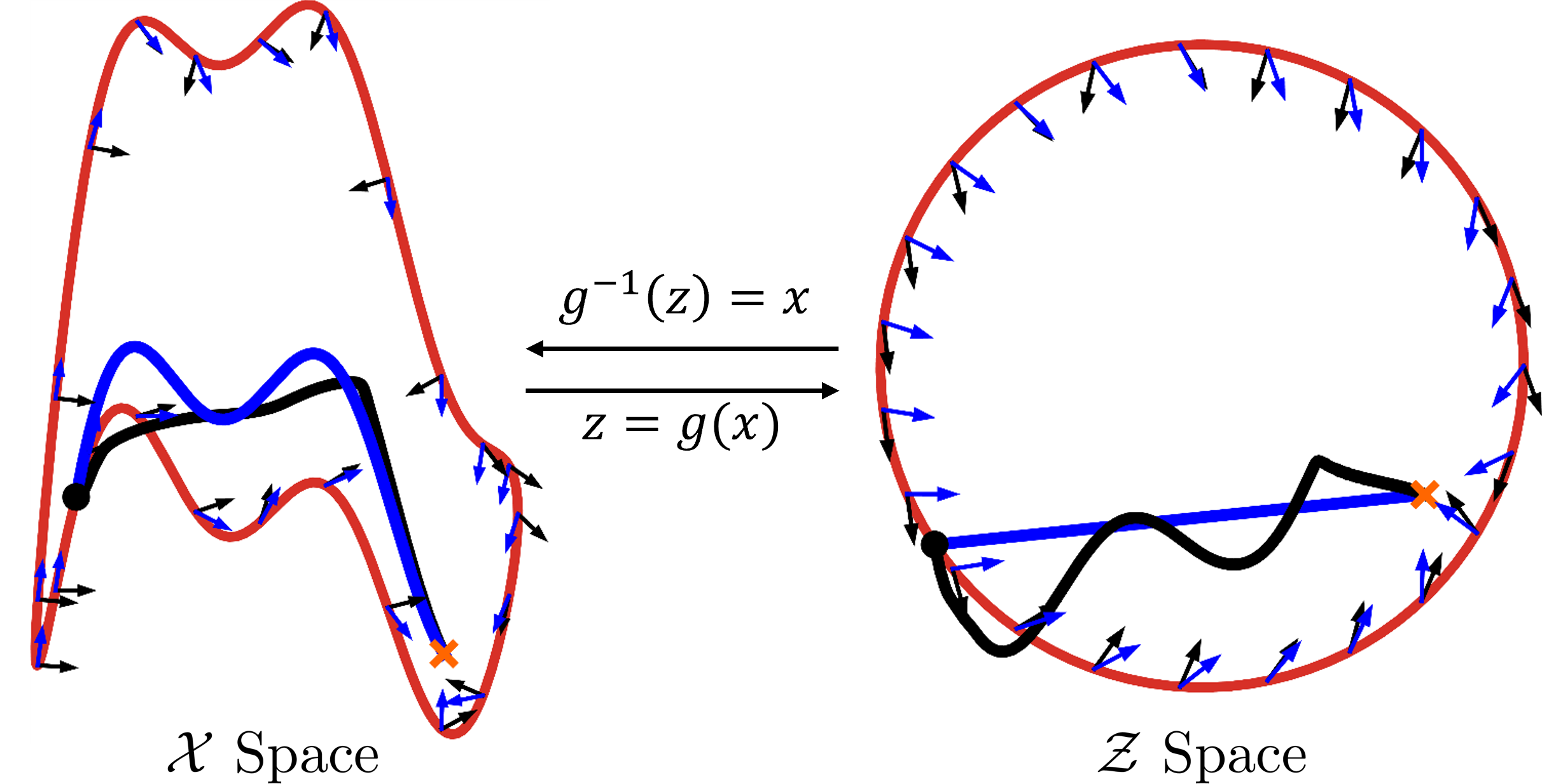} \\
        (b) $\xgoal\neq (0, 0)$
    \end{tabular}
    \caption{Trajectory samples and vector field on the boundary for the natural gradient flow \eqref{eq:natural_gradient} (blue) and the gradient flow \eqref{eq:grad-flow} (black) with different goal points, where red curves are the boundaries. }
    \label{fig:ngd-vs-gd}
\end{figure}

When $\xgoal=(0,0)$, we have that $V(x,\xgoal)$ is a barrier function for both \eqref{eq:natural_gradient} and \eqref{eq:grad-flow}. Thus, $\Xf$ is forward invariant in both cases, see Fig.~\ref{fig:ngd-vs-gd}(a). When $\xgoal$ changes, it is no longer a barrier function as $V(x,\xgoal)$ is not a constant for $x\in \partial \Xf$. Fig.~\ref{fig:ngd-vs-gd}(b) shows that $\Xf$ is no longer a forward-invariant set for \eqref{eq:grad-flow}. This is also supported by the fact that its vector field points outward at some part of the boundary. For the proposed approach, $\Xf$ is forward invariant as the vector field of \eqref{eq:natural_gradient} always points inward for all $\xgoal\in \Xf$. A barrier function $h(x)$ can be constructed via \eqref{eq:barrier}.

\end{example}

\subsection{Comparison with Navigation Function based Approach}

When the goal state $\xgoal$ is fixed, a classical approach \cite{koditschek1990robot} to construct system \eqref{eq:system} is via gradient flow
\begin{equation}\label{eq:navigation-func}
    \dot{x} = -\nabla \phi(x)
\end{equation} 
where $\phi : \Xf \to [0,1]$ is a \emph{navigation function} satisfying the following conditions: 
\begin{itemize}
    \item[N1)] $\phi$ is Morse function (i.e., $\phi$ is smooth and it has no degenerate critical point);
    \item[N2)] $\phi$ has a unique minimum on $\Xf$ at $\xgoal$ and no other critical points;
    \item[N3)] $\nabla\phi$ is bounded on $\Xf$; 
    \item[N4)] $\phi(x)=1$ for all $x\in \partial \Xf$.
\end{itemize}
The navigation function $\phi$ serves as both a Lyapunov function and a barrier function since 
\[
\dot{\phi}=-|\nabla \phi(x)|^2 <0, \quad \forall x \in \Xf, \;x\neq \xgoal.
\]
Note that $\phi(x)=|g(x)|^2$ with $g:\Xf\rightarrow \mathcal{B}^n$ and $g(\xgoal)=0$ is a valid navigation function. However, one needs to recompute $g$ when $\xgoal$ changes. 

Compared with the navigation based approach, our method is more flexible as it does not require recomputing $g$ when $\xgoal$ changes since it uses different certificate functions for stability and safety, although both are expressed in terms of a single learned diffeomorphism $g$.

\subsection{Compared with Existing Diffeomorphism based Dynamical Approaches}

The diffeomorphism based approach has also been recently explored for learning stable neural ODE from demonstration, see \cite{rana2020euclideanizing,pmlr-v162-zhi22a,pmlr-v168-zhi22a}. Specifically, those approaches take the following gradient flow in the $\Z$-space:
\begin{equation}\label{eq:z-gen}
    \dot{z}=-\nabla \Phi(z)
\end{equation}
where the potential function $\Phi$ is positive definite, convex, continuously differentiable,
and radially unbounded. And a natural gradient dynamics in the $\X$-space is constructed by pulling \eqref{eq:z-gen}  back to the $\Xf$-space via a diffeomorphism $g$. The primary goal of those approaches is to learn stable dynamics that mimics the demonstration data. 

Different from those approaches, our method can learn both stable and safe dynamics from data. The second difference is that our approach can generalize to unseen goal point without retraining the model. Finally, our approach imposes explicit bounds on the diffeomorphism $g$, which can be seen as effective regularization preventing overfitting. Meanwhile, explicit bounds on tracking error and vector field magnitude can be obtained, which is useful for practical applications.

\section{Learning an All-Pairs Motion Planner}\label{sec:learning}

In this section, we aim to translate the above theoretical construction into a tractable machine learning setup, detailing the choice of training data, model class, and loss functions. 

We parameterize the diffeomorphism $g$ by some smooth bi-Lipschitz neural network $g_\theta:\R^n\to\R^n$ with $\theta\in \R^p$ as the learnable parameter.
By construction, system \eqref{eq:natural_gradient} is smooth and globally equilibrium-independent exponentially stable, as shown in Theorem~\ref{thm:main}. To ensure safety, one needs to learn a $g_\theta$ that can be used to characterize the set $\Xf$. Since a candidate Lyapunov function $V(x,\xgoal)$ is used in the model construct, we seek to separate the safe and unsafe datasets in \eqref{eq:datasets} via a  Lyapunov sublevel set of $V$. Specifically, we pick up a point $\hat{x}_{\star} \in \Df $ as the goal and take $g_\theta(\hat{x}_{\star})=0$. Then, the Lyapunov function in \eqref{eq:v-func} can be written as
\begin{equation}\label{eq:vtx}
    \tilde{V}_\theta(x)=V(x,\hat{x}_\star)=\frac{\lambda}{2}|g_\theta(x)|^2,
\end{equation}
whose sublevel sets are $\Omega_{\theta}^{c}=\{x: \tilde{V}_\theta(x)\leq c\}$ with $c>0$. Note that $\Omega_{\theta}^{c}$ is diffeomorphic to $\mathcal{B}^n$ for any $\theta\in \R^p$ and $c>0$. Now, the learning problem becomes: find a pair $(c,\theta)$ such that 
\begin{subequations}\label{eq:levelset}
    \begin{align}
        x_i\in \Df\; &\Rightarrow\; x_i\in \Omega_\theta^c, \label{eq:levelset-safe}\\
        x_j\in \Do\; &\Rightarrow \; x_j \notin \Omega_\theta^c \label{eq:levelset-unsafe}.
    \end{align}
\end{subequations}
\subsection{Training Data}
To solve the above learning problem, we need to assign labels to the points from the datasets $\Df$ and $\Do$. An intuitive approach is to associate $x_i\in \Df$ and $x_j\in \Do$ with labels of 0 and 1, respectively. The learning problem in \eqref{eq:levelset} is formulated as a classification task, where $\tilde{V}_\theta(x)$ is the classifier. However, those labels do not provide informative geometric information in $\Df$ and $\Do$. 

By leveraging the existing sampling-based motion planning algorithms (e.g. PRM \cite{Kavraki1996ProbabilisticRF} or RRT \cite{LaValle1998RapidlyexploringRT}), we can assign each point $x_i\in\Df$ with a label $c_i$ indicating the shortest path length from $x_i$ to the target $\hat{x}_{\star}$. Specifically, we first construct a graph by connecting each $x_i$ to a set of its nearby neighbors in $\Df$. From this graph, we can define the cost-to-go function $d:\mathcal{D}_{\text{safe}}\to\R_{\geq0}$ as the shortest path distance from sample $x\in \mathcal{D}_{\text{safe}}$ to the goal $\hat{x}_\star$, which is a proxy for the distance between $x_i$ and $\hat{x}_{\star}$. The function $d(x)$ naturally reflects the geometry of $\Xf$: samples near the goal attain small values, while samples that are distant or geometrically separated from $\hat{x}_\star$ attain large values. Thus, $d(x)$ provides meaningful information for training $\tilde{V}_\theta(x)$. Then, we construct the training datasets as follows:
\begin{subequations}
    \begin{align}
        \bar{\mathcal{D}}_{\text{safe}}&=\{(x_i,c_i): c_i=d(x_i),\, x_i\in \Df\}\label{eq:d_safe} \\
        \bar{\mathcal{D}}_{\text{unsafe}}&=\{(x_j,c_j): c_j=\bar{c}+\delta,\,x_j\in \Do\}\label{eq:d_unsafe}
    \end{align}
\end{subequations}
where $\bar{c}$ is the maximum label value in $\bar{\mathcal{D}}_{\text{safe}}$, and $\delta>0$ is a hyperparameter which ensures that there exists $c\in [\bar{c},\bar{c}+\delta]$ satisfying \eqref{eq:levelset}.

\subsection{Model Class}
In this work, we use the BiLipNet \cite{wang2024monotone} as the model class for $g_\theta$. BiLipNets can enforce certified bi-Lipschitz bounds $\mu$ and $\nu$ via a method derived from \cite{wang2023direct} which are, to the authors' knowledge, the tightest available. The bi-Lipschitz bounds are trainable parameters, and their ratio $\tfrac{\nu}{\mu}$ can be considered a tunable \textit{distortion} parameter, describing how much the learnt representation of $\Xf$ distorts from a unit ball, and therefore how much the learnt trajectories can deviate from straight lines -- notice that this also appears in the overshoot constant for our exponential convergence bound \eqref{eq:x-exp-stable}. The lower bound $\mu$ also appears in our bound on velocity \eqref{eq:vel_bound}.

BiLipNets have a number of other advantages: firstly, BiLipNets admits a direct model parameterization, which allows training within the standard unconstrained optimization methods such as stochastic gradient descent. Secondly, the feedthrough layer architecture can improve the model expressivity without suffering from vanishing gradients. Thirdly, BiLipNets have  a structure that admits fast splitting-based solvers for computing the model inverse.

\subsection{Loss Function}
The loss function will generally include two components. The first component trains the diffeomorphism to map the safe set $\Xf$ onto the unit ball. However, this leaves substantial flexibility for the shape of the mapping \textit{inside} $\Xf$, so a second task-specific loss term can be employed which may take many forms, e.g. training the dynamics \eqref{eq:natural_gradient} to mimic demonstration trajectories.

For the first task, i.e., achieving \eqref{eq:levelset}, we choose the following loss function 
\begin{equation}\label{eq: loss base}
    \mathcal{L}(\theta)=\mathcal{L}_{\text{safe}}(\theta)+\mathcal{L}_{\text{unsafe}}(\theta)
\end{equation}
where
\begin{equation*}
    \begin{split}
        \mathcal{L}_{\text{safe}}(\theta)&=\frac{1}{|\Bar{\mathcal{D}}_{\text{safe}}|}\sum \max\bigl(\tilde{V}_\theta(x_i)-c_i, 0\bigr)^2, \\
        \mathcal{L}_{\text{unsafe}}(\theta)&=\frac{1}{|\Bar{\mathcal{D}}_{\text{unsafe}}|}\sum \max\bigl(c_j-\tilde{V}_\theta(x_j), 0\bigr)^2.
    \end{split}
\end{equation*}
The term $\mathcal{L}_{\text{safe}}$ penalizes the samples from $\Bar{\mathcal{D}}_{\text{safe}}$ for which $\tilde{V}_\theta(x_i)>c_i$ is required to ensure \eqref{eq:levelset-safe}, while $\mathcal{L}_{\text{unsafe}}$ penalizes the samples from $\Bar{\mathcal{D}}_{\text{unsafe}}$ for which $\tilde{V}_\theta(x_j)<\Bar{c}$ to ensure \eqref{eq:levelset-unsafe}. 

The loss function for the second task may take various forms. E.g., when the demonstration dataset is available, we can define 
\begin{equation*}
    \mathcal{L}_{\text{task}}(\theta)=\frac{1}{|\mathcal{D}_{\text{demo}}|}\sum\bigr(  \Dot{x}_k - f_\theta(x_k, x_{k,\star}) \bigl)^2
\end{equation*}
where $f_\theta$ is the vector field in \eqref{eq:natural_gradient} with $g_\theta$. The total loss is taken as $\mathcal{L}_{t}=\mathcal{L}+\rho\mathcal{L}_{\text{task}}$ with weighting $\rho>0$.

\section{Numerical Experiments}\label{sec:numeric experiment}

We illustrate the proposed approach on a 2D corridor navigation task (see Fig.~\ref{fig:rrt}), where we aim to generate safe and smooth trajectories from any initial configuration $x_0\in\Xf$ to any goal $\xgoal\in\Xf$ in the presence of geometric obstacles. All experiments are implemented in Python using JAX and executed on an NVIDIA RTX 4090 GPU.
Code is available at \url{https://github.com/acfr/Goal-Conditioned-Safe-ODE}.

\subsection{Data Generation and Training details}
As shown in Fig.~\ref{fig:rrt} (left), we initialized RRT \cite{LaValle1998RapidlyexploringRT} at a fixed goal $\hat{x}_\star$ to generate a shortest-path tree over $\Xf$. This automatically provides the dataset pair $(x_i,d(x_i))$ where $d(x_i)$ is the cost-to-go with $x_i\in\Xf$, see Fig.~\ref{fig:rrt} (right). We take 2,500 samples to formulate the dataset $\bar{\mathcal{D}}_{\text{safe}}$ in \eqref{eq:d_safe}. Another 2,500 samples are uniformly sampled in $\Xo$, which forms $\bar{\mathcal{D}}_{\text{unsafe}}$ in \eqref{eq:d_unsafe}.  
\begin{figure}[!tb]
    \centering
    \includegraphics[width=\linewidth]{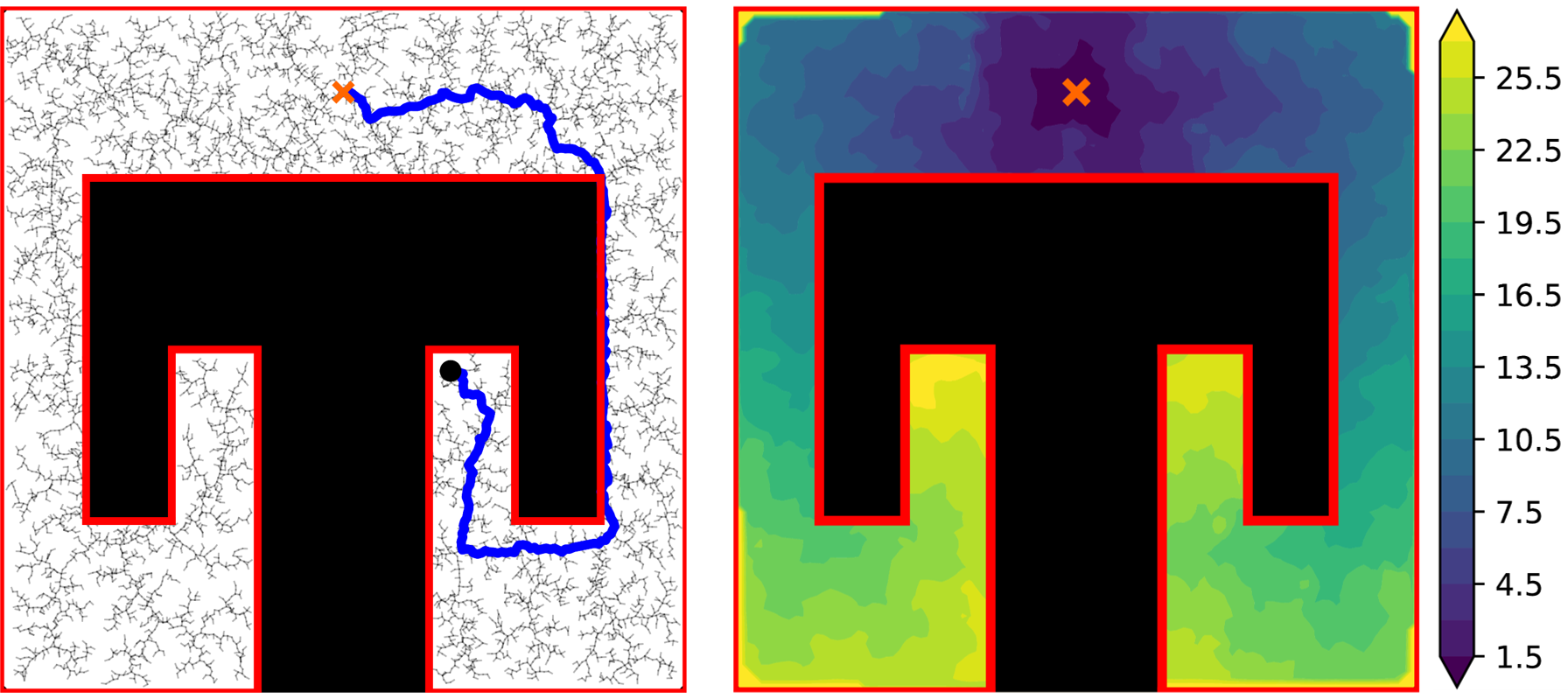}
    \caption{
    RRT data (gray) in the corridor environment. (Left) RRT rooted at $\hat{x}_{\star}$, with a representative trajectory (blue) in $\X$. (Right) Corresponding cost-to-go field $d(\cdot)$ visualized via contour lines over $\Xf$.}\label{fig:rrt}
\end{figure} 

We use BiLipNet from~\cite{wang2024monotone} to parameterize the bi-Lipschitz diffeomorphism $g_\theta$. 
The network is trained based on the loss function in \eqref{eq: loss base} via the Adam optimizer \cite{kingma2015adam} with a batch size of 16 for 1500 epochs.

\subsection{Results and Discussions}

Fig. \ref{fig:boundary_mapping} shows the learned mapping $g$ that transforms $\Xf'\subset \Xf$ in the $\X$-space (Left) to a unit ball $\mathcal{B}^2$ in the $\Z$-space (Right). $\partial \Xf'$ and $\partial \mathcal{B}^2$ are indicated by red curves while $\partial \Xf$ is in black. $\partial \Xf'$ conforms to the geometry of the obstacle boundaries, which is neither convex nor star-convex.

Fig. \ref{fig:natural_gradient_flow} shows that complex trajectories of the natural gradient flow system \eqref{eq:natural_gradient} (left) are equivariant to linear trajectories of system \eqref{eq:z_dynamics} (right) under the coordinate change. Pulling back straight lines in $\Z$-space (Fig. \ref{fig:natural_gradient_flow} right) through $g_{\theta}^{-1}$ yields safe trajectories in $\Xf$ (Fig. \ref{fig:natural_gradient_flow} left) that respect the obstacle geometry. All trajectories in Fig. \ref{fig:natural_gradient_flow} (left) also converge to $\hat{x}_\star$.

\begin{figure}[!tb]
    \centering
    \begin{minipage}[c]{0.99\linewidth}
        \begin{minipage}[c]{0.46\linewidth}
            \centering
            \includegraphics[width=\linewidth]{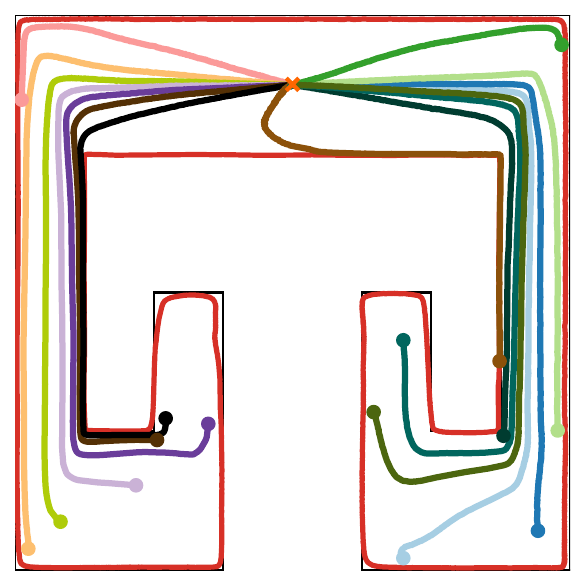}
        \end{minipage}
        \hfill
        \begin{minipage}[c]{0.52\linewidth}
            \centering
            \includegraphics[width=\linewidth]{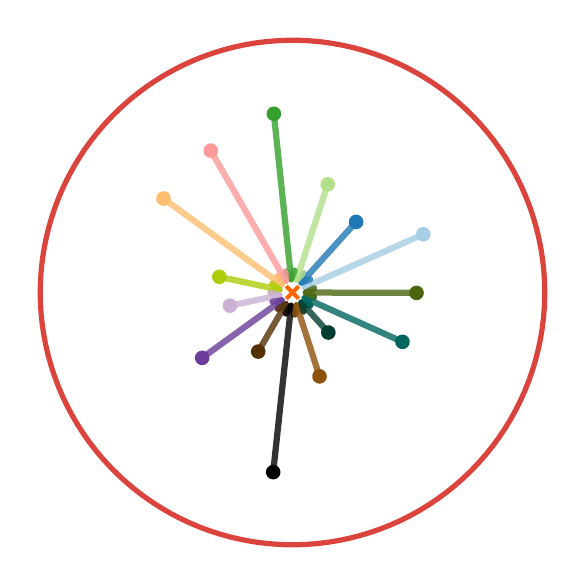}
        \end{minipage}
    \end{minipage}
    \begin{minipage}[c]{0.99\linewidth}
        \begin{minipage}[c]{0.46\linewidth}
            \centering
            \small Trajectories in $\X$ Space
        \end{minipage}
        \hfill
        \begin{minipage}[c]{0.52\linewidth}
            \centering
            \small Trajectories in $\Z$ Space
        \end{minipage}
    \end{minipage}
    \centering
    \caption{
    Trajectories generated by system \eqref{eq:natural_gradient} from multiple initial configurations to the goal $\hat{x}_\star$ in the training dataset. (Left) Smooth, safe paths in the $\X$-space. (Right) Those paths are transformed into straight-line trajectories in the $\Z$-space, demonstrating the geometric simplification induced by $g_{\theta}$.}
    \label{fig:natural_gradient_flow}
\end{figure} 
Fig.~\ref{fig:NGF-varying-goal} illustrates that multiple trajectories converge to a distinct, previously unseen goal $x_\star$ (red cross). This indicates system \eqref{eq:natural_gradient} is equilibrium-independent stable and safe. Note that the model was trained using data corresponding to a \textit{single} goal, but it generalizes gracefully to a goal in a completely different location.

\begin{figure}[!tb]
    \centering
    \begin{minipage}[c]{0.99\linewidth}
        \begin{minipage}[c]{0.46\linewidth}
            \centering
            \includegraphics[width=\linewidth]{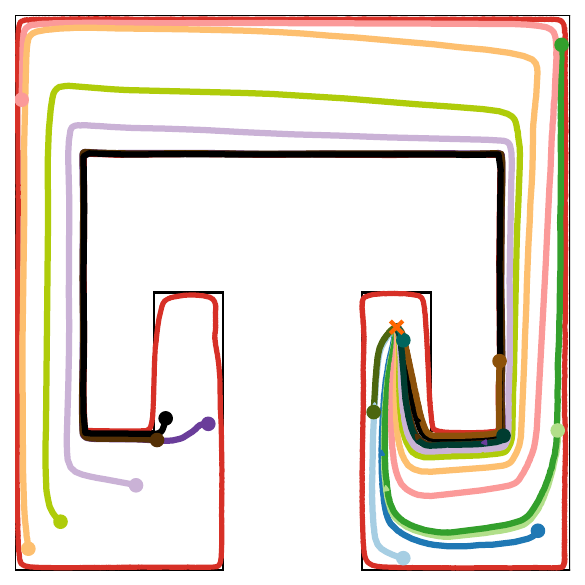}
        \end{minipage}
        \hfill
        \begin{minipage}[c]{0.52\linewidth}
            \centering
            \includegraphics[width=\linewidth]{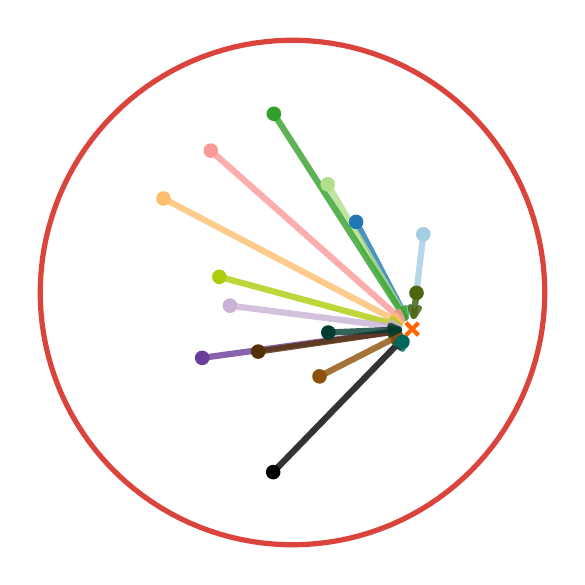}
        \end{minipage}
    \end{minipage}
    \begin{minipage}[c]{0.99\linewidth}
        \begin{minipage}[c]{0.46\linewidth}
            \centering
            \small Trajectories in $\X$ Space
        \end{minipage}
        \hfill
        \begin{minipage}[c]{0.52\linewidth}
            \centering
            \small Trajectories in $\Z$ Space
        \end{minipage}
    \end{minipage}
    \centering
    \caption{
    Generalization to previously unseen goal $\xgoal$ without retraining. (Left) Smooth, safe paths in the $\X$-space converging to a new goal $\xgoal$ (red cross) from multiple initial configurations. (Right) In the $\Z$-space, the paths are transformed into straight-line trajectories within the unit ball.
    }\label{fig:NGF-varying-goal}
\end{figure} 

\section{Conclusion}\label{sec:conclusion}
In this paper, we presented a learning-based approach for safe, stable, and smooth all-pairs motion planning. Our approach uses a bi-Lipschitz diffeomorphism to transform a geometrically complex safe set into the unit ball, and complex motions within this set into simple linear stable dynamics which corresponds to a goal-conditioned natural gradient in the original state space. This approach guarantees that both safety and exponential stability are preserved regardless of the goal location within the safe set. Empirical results in 2D corridor navigation task illustrate the the proposed approach. 

\bibliographystyle{ieeetr}
\bibliography{ref}

\end{document}